\documentclass{article}

\usepackage{amsmath,amssymb,amsthm}
\usepackage{color}
\usepackage{graphicx}
\usepackage[ruled,vlined]{algorithm2e}
\usepackage{algpseudocode}
\usepackage{ijcai17}
\usepackage{balance}
\usepackage{times}

\newtheorem{proposition}{Proposition}

\newcommand{\MKL}[2]{\mathcal{M}^{#1}_{#2}}
\newcommand{\mc}{\mathcal}
\newcommand{\tsc}[1]{\mbox{\large\textsc{#1}}}
\newcommand{\tuple}[1]{\langle #1 \rangle}

\title{Generalizing the Role of Determinization in Probabilistic Planning}
\author{Luis Pineda \and Shlomo Zilberstein \\ 
College of Information and Computer Sciences \\ 
University of Massachusetts, Amherst, MA 01003\\
\{lpineda, shlomo\}{@cs.umass.edu}}

\begin{document}

\maketitle

\begin{abstract}
The stochastic shortest path problem (SSP) is a highly expressive model for probabilistic planning.  The computational hardness of SSPs has sparked interest in determinization-based planners that can quickly solve large problems.  However, existing methods employ a simplistic approach to determinization.  In particular, they ignore the possibility of tailoring the determinization to the specific characteristics of the target domain.  In this work we examine this question, by showing that learning a good determinization for a planning domain can be done efficiently and can improve performance.  Moreover, we show how to directly incorporate probabilistic reasoning into the planning problem when a good determinization is not sufficient by itself.  Based on these insights, we introduce a planner, FF-LAO*, that outperforms state-of-the-art probabilistic planners on several well-known competition benchmarks.

\end{abstract}

\section{Introduction}
One of the most popular models for probabilistic planning is the Stochastic Shortest Path SSP)~\cite{bertsekas1991analysis}. A solution to an SSP is a sequence of actions that, starting from a given initial state, minimizes the expected cost of reaching a state from a given set of goal states. The uncertainty associated with the outcome of each action induces probabilistic transitions between states. This model has been used for a wide variety of applications, such as route planning in the presence of traffic delays~\cite{lim2013practical}, quantifying the value of battery energy storage systems~\cite{tan2015stochastic}, modeling wildfire propagation~\cite{hajian2016modeling}, and semi-autonomous driving~\cite{wray2016hierarchical}. 

Despite its popularity, a major obstacle to the wide applicability of SSPs is the prohibitive computational cost of finding an optimal solution. This is due to the fact that optimal policies can, in the worst case, cover a number of states linear in the size of the state space, which in turn can be exponential in the number of variables describing the problem. 

Given the difficulty of solving SSPs optimally, there has been much interest in developing methods that sacrifice optimality for the sake of computational efficiency. Among these methods, one of the most successful approaches has been the use of \emph{determinization} complemented by replanning. These methods became popular thanks to the unexpected success of FF-Replan~\cite{yoon2007ff} in the International Probabilistic Planning Competition (IPPC)~\cite{bryce20086th}. FF-Replan works by transforming the original problem into a deterministic one (e.g., by only considering most probable outcomes), and using the FF classical planner~\cite{hoffmann2001ff} to solve the resulting problems. This process is then repeated during execution whenever a state not covered by the current deterministic plan is observed.

Although FF-Replan is significantly faster than any optimal probabilistic planner, it performs poorly in many problem domains due to the fact that it ignores the uncertainty. This has led to the development of more robust determinization-based algorithms, that incorporate some form of probabilistic reasoning. 

RFF~\cite{teichteil2010incremental}---the winner of IPPC'08---works by incrementally aggregating partial plans until the result covers a high probability envelop of states; each partial plan is computed using determinization and the FF planner. FF-Hindsight~\cite{yoon2008probabilistic} works by sampling a set of deterministic ``futures'' of the original problem, solves each using FF, and combines their cost to estimate the costs in the original problem. HMDPP~\cite{keyder2008hmdpp} introduces a self-loop determinization trick that nudges the deterministic planner into generating plans with low probability of deviation. SSiPP-FF~\cite{trevizan2014depth} works by creating short-sighted problems that consider only states up to a certain horizon from the current state. These smaller problems are solved optimally, and when a tip state is found during execution, the FF-Replan method is used.

Although these planners are capable of quickly solving large problems, they all employ a rather crude approach to determinization, by generally relying on either the most-likely-outcome (MLO) determinization, or the all-outcomes (AO) one. Using MLO can make planners ignore important sections of the state space (e.g., outcomes leading to dead-ends), resulting in poor policies. Using AO has other problems such as potentially wasting computation on irrelevant/unlikely sections of the state space, or treating all paths to the goal as equally important. 

Recent work has shown that the choice of determinization can sometimes have a significant impact in the quality of a determinization-based planning approach~\cite{PZicaps14}. In fact, even in some domains deemed ``probabilistically interesting''~\cite{little2007probabilistic} planning with a good determinization can actually result in optimal plans for the original SSP (e.g., the triangle tireworld domain).
However, not much work has been done on generating methods for choosing a determinization that works well for a particular domain.

In this work, we address these issues with two main contributions.
First, we present a new planner, FF-LAO*, that combines the LAO* optimal SSP solver~\cite{HZaij01} with the FF classical planner. FF-LAO* can leverage fast deterministic planning to estimate state values, but still partially reason about the complete probabilistic model if so-desired; it does this by relying on the \emph{reduced models} framework introduced by~\citeauthor{PZicaps14}~[\citeyear{PZicaps14}]. 

Our second contribution is showing that it is possible to learn a good determinization on small instances of a planning domain, such that, when applied to larger instances, significant gains in efficiency and performance can be realized. This is the first work that selects a determinization based on the anticipated performance in the original probabilistic domain.

The rest of the paper is structured as follows: Section~\ref{sec:background} gives background on SSPs and reduced models, Section~\ref{sec:fflao} describes the FF-LAO* algorithm, Section~\ref{sec:choosing-red} explains how to choose a good determinization for a particular planning domain, Section~\ref{sec:experiments} presents experimental results, and Section~\ref{sec:conclusion} summarizes the conclusions and ideas for future work. 

\section{Background}
\label{sec:background}
A Stochastic Shortest Path (SSP) problem~\cite{bertsekas1991analysis} is defined by a tuple $\tuple{S, A, T, C, s_0, G}$, where $S$ is a finite set of states, $A$ is a finite set of actions, $T(s'|s,a) \in [0,1]$ represents the probability of reaching state $s'$ when action $a$ is taken in state $s$, $C(s,a) \in [0,\infty)$ is the cost of applying action $a$ in state $s$, $s_0$ is an initial state and $G$ is a set of goal states satisfying $\forall s_g \in G, a \in A,~T(s_g|s_g, a) = 1 \wedge C(s_g,a)=0$. Moreover, we assume costs satisfy $\forall s \in S\!\setminus\!G,  C(s,a)>0$.
Interestingly, SSPs are a variant of Markov Decision Process (MDPs)~\cite{Puterman94}  that has been shown to be more general than finite-horizon and infinite-horizon discounted MDPs~\cite{bertsekas1995neuro}.

A solution to an SSP is a \emph{policy}, a mapping $\pi:S \rightarrow A$, indicating that action $\pi(s)$ should be taken at state $s$. A policy $\pi$ induces a value function $V^{\pi}:S \rightarrow \mathbb{R}$ that represents the expected cumulative cost of reaching $s_g \in G$ by following policy $\pi$ from state $s$. An optimal policy $\pi^*$ is one that minimizes this expected cumulative cost; similarly, we use the notation $V^*$ to refer to the optimal value function.

For an SSP to be well-defined, a policy must exist such that the goal is reachable from any state with probability 1. Under this assumption, an SSP is guaranteed to have an optimal solution, and the optimal value function is unique. This optimal value function can then be found as the fixed point of the Bellman update operator (Eq. \ref{eq:bellman}).
\begin{equation}
V(s) = \min_{a \in A} \Big \{ C(s,a) + \sum_{s' \in S} T(s'|s,a)V(s') \Big \}
\label{eq:bellman}
\end{equation}

There is a variety of languages to compactly describe SSPs, of which PPDDL~\cite{younes2004ppddl1} has been most widely used within the AI community. As it turns out, it has been shown that finding whether a plan exists in a compactly described problem is EXPTIME-complete~\cite{littman1997probabilistic}. As mentioned earlier, this challenging complexity has led to the development of many approximate methods for solving SSPs. Particularly relevant to our work are determinization-based approaches, and, more generally, the reduced models framework.

\subsection{Reduced Models}
Given an SSP, the reduced models framework creates a simplified model characterized by two parameters: the number of outcomes per action that are fully accounted for (referred to as \emph{primary} outcomes), and the maximum number of occurrences of the remaining outcomes that are planned for in advance (referred to as \emph{exceptions}). This type of reduction generalizes single-outcome determinization, and introduces a spectrum of reductions with varying levels of probabilistic 
complexity.

The model assumes a factored representation of SSPs in which actions are probabilistic operators of the form:
\[ a = \tuple{\textit{preconditions},\textit{cost}, [p_1 : e_1, ..., p_m : e_m]}, \]
where each effect $e_i$, for $i \in \{1,...,m\}$, is associated with a probability $p_i$ of occurring when $a$ is executed, and there exists a successor function, $\tau$, that maps effects to successor states, so that $s'=\tau(s,e_i)$ and $T(s'|s,a) = p_i$.

Using this formalism, a \emph{reduced model} of an SSP $\tuple{S, A, T, C, s_0, G}$ is defined as follows:
\begin{itemize}
	\item The set of states is defined as $S' = S \times \{0,1,..,k\}$, where $k$ is a positive integer; 
	\item The set of actions is the original set, $A$;
	\item The transition function is defined by Eq.~(\ref{eq:transition_jlessk});
	\item The cost function is defined as $C'((s,j),a) = C(s,a)$, for all $(s,j) \in S' \wedge a\ \in A$;
	\item The initial state is $(s_0,0)$;
	\item The set of goals is defined as $G' = \{(s,j) \in S' | s \in G \}$.
\end{itemize}

The transition function defined below, while seemingly complicated, describes a simple process. The value $j$ in state $(s,j)$ keeps counts of how many exceptions have occurred up to that point in execution. For states where the count is less than the \emph{exception bound}, $k$, the transition function operates as the original, except that the counter is increased by one for successors labeled as exceptions. On the other hand, for states with count $j=k$, the transition completely ignores exceptions, and only considers transitions to primary outcomes, redistributing the ignored probabilities so that they form a proper distribution (e.g., by normalizing). The notation $\mc{P}_a$ used below refers to the set of primary effects. 

\begin{multline}
\label{eq:transition_jlessk}
T'((s',j') | (s,j),a) = \\ 
\begin{cases}	
p_i & \mathrm{if} ~ j<k \wedge j'=j \wedge e_i \in \mc{P}_a \\
p_i & \mathrm{if} ~ j<k \wedge j'=j+1 \wedge e_i \notin \mc{P}_a \\
p'_i & \mathrm{if} ~ j=j'=k \wedge e_i \in \mc{P}_a\\
0 &  \mathrm{if} ~ j=j'=k \wedge e_i \notin \mc{P}_a
\end{cases}
\end{multline}where $s' = \tau(s,e_i)$ and the set $\{p'_1, ..., p'_m \}$ is any set of real numbers that satisfy
\begin{equation}
\forall i: e_i \in \mc{P}_a ~~ p'_i > 0 ~~~\mathrm{and}~~~
\sum\limits_{i : e_i \in \mc{P}_a} p'_i = 1
\end{equation}

Note that the reduced models framework 
encapsulates single-outcome determinization, which is simply a reduction where the set of primary outcomes has size 1 and the value of \mbox{$k=0$}. In this work we are indeed concerned with reductions having a single primary outcome, but will include the possibility of using $k>0$. We refer to these models as $\MKL{k}{1}$-reductions.

\section{FF-LAO*}
\label{sec:fflao}
Reducing an SSP can significantly accelerate planning times by pruning large sections of the state space. However, there are many domains in which solving a reduced model optimally is still prohibitively expensive.
In fact, for complex domains like the ones used in IPPC~\cite{bryce20086th}, using determinization and $k=0$ already results in problems too large to be solved optimally in a practical manner. 

To address this issue, we present a planner that combines the flexibility of the reduced models framework with the efficiency of a classical planner. We call this planner  \tsc{ff-lao*}, as it is an extension of the \tsc{lao*} algorithm~\cite{HZaij01} that leverages the \tsc{ff} classical planner~\cite{hoffmann2001ff} to accelerate computation.  

\tsc{ff-lao*} (Algorithms~\ref{alg:ff-lao}-\ref{alg:ff-bellman}) receives as input an \mbox{$\MKL{k}{1}$-reduction}, \mbox{$\mathbb{S}'\!=\!\tuple{S', A, T', C', (s_0, 0), G'}$}; i.e., one that becomes deterministic after the exception bound is reached (equivalently, one where \mbox{$\forall a \in A, |\mc{P}_a| = 1$}). The remaining inputs are the exception bound, $k$, and the error tolerance, $\epsilon$. We use~~$\mathbb{S}$ to denote the original SSP from which $\mathbb{S}'$ is derived.

\tsc{ff-lao*} works almost exactly as \tsc{lao*}, except that \tsc{ff} is used to compute values and actions for states that have reached the exception bound (i.e., states of the form $(s,k)$). This occurs in lines~4 and~8 of Algorithm \ref{alg:ff-lao}, where the state expansion and test convergence procedures are replaced with versions that use FF (Algorithms~\ref{alg:ff-expand} and~\ref{alg:ff-test}, respectively). 

Readers familiar with \tsc{lao*} may notice differences with respect to the usual expansion and convergence test procedures. In particular, note the inclusion of \emph{if} statements in line~7 (both procedures), where the successors of the expanded state are only added to the stack if $j<k$. The reason is that states $(s,k)$ will be solved by calling FF, so there is no need to expand their successors. 

It is possible, of course, to remove these \emph{if} statements and let \tsc{ff-lao*} continue the search; in that case, \tsc{ff} will be used as an inadmissible heuristic. However, this does not improve the theoretical properties of the algorithm (neither version is optimal), and results in higher computation times, so we prefer the version shown in the pseudocode.

The actual call to FF is done in Algorithm \ref{alg:ff-bellman} (\tsc{ff-bellman-update}). This procedure performs a Bellman update (Eq. \ref{eq:bellman}) for any state $(s,j)$ with $j < k$, and stores the updated cost estimate and best action in global variables $V[(s,j)]$ and $\pi[(s,j)]$, respectively (lines~6-7). For simplicity of presentation, we use the following action-value function:

\begin{smaller}
	\begin{equation*}
	Q((s, j), a) \equiv C'((s,j),a) + \sum_{(s',j')} T'((s',j')|(s,j),a)V[(s',j')]
	\end{equation*}
\end{smaller}and assume, as is common for heuristic search algorithms, that the values $V[(s',j')]$ are initialized using an admissible heuristic for $\mathbb{S}'$.

For states $(s,k)$, the \tsc{ff-bellman-update} procedure creates a PDDL file\footnote{In practice, we create the PDDL file representing $M$ before calling \tsc{ff-lao*} and store its name in memory. \tsc{create-pddl} is shown for simplicity of presentation.}, denoted as $D$,  representing the deterministic problem induced by $M$ when $j=k$, with initial state $s$ (\tsc{create-pddl} in line 3). The procedure then calls \tsc{ff} with input $D$ (line 4) and memoizes costs and actions for all the states visited in the plan computed by \tsc{FF} (lines 5-7). More concretely, for each state $s_i$ visited by this plan, we set $V[(s_i,k)]$ to be the cost, according to $C'$, of the plan computed by \tsc{FF} for that state (line 6), and set $\pi[(s_i,k)]$ to be the corresponding action (line 7). Additionally, note that the estimates $V[(s,k)]$ are not admissible, even with respect to the input $\MKL{k}{1}$-reduction, since $\tsc{FF}$ is not an optimal planner for deterministic problems. Finally, in the case that \tsc{ff} returns failure, we set $V[(s,k)]=\infty$ and $\pi[(s,k)]=\textrm{NOP}$. 

\tsc{ff-bellman-update} also returns the residual, defined as the absolute difference between the previous cost estimate, and the estimate after applying the Bellman equation. This residual is used by \tsc{ff-test-convergence} to check the stopping criterion of the algorithm.

\begin{algorithm}[t]
	\smaller
	\textbf{input}: $\mathbb{S}'\!=\!\tuple{S', A, T', C', (s_0, 0), G'}$, $k$, $\epsilon$ \\
	\nl \While{\textbf{true}} {
		// \textit{Node expansion step} \\
		\nl \While{\textbf{true}} {
			\nl \textit{visited} $\leftarrow \emptyset$ \\
			\nl \textit{cnt} $\leftarrow$ \tsc{ff-expand}$\big(\mathbb{S}', (s,j), k, \textit{visited}\big)$ \\
			\nl \If {$ \textit{cnt} = 0$} {
				// \textit{No tip nodes were expanded, so current policy is closed} \\
				\textbf{break} \\
			}
		}
		// \textit{Convergence test step} \\
		\nl \While{\textbf{true}} {
			\nl \textit{visited} $\leftarrow \emptyset$ \\
			\nl \textit{error} $\leftarrow$ \tsc{ff-test-convergence}$\big(\mathbb{S}', (s,j), k, \textit{visited}\big)$ \\
			\nl \If {$ \textit{error} < \epsilon$} {\textbf{return} // \textit{solution found}\\}
			\nl \If {$ \textit{error} = \infty$} {\textbf{break} // \textit{change in partial policy, go back to expansion step} \\}
		}
	}
	\caption{\tsc{ff-lao*}}
	\label{alg:ff-lao}
\end{algorithm} 

\begin{algorithm}[t]
	\smaller
	\textbf{input}: $\mathbb{S}'\!=\!\tuple{S', A, T', C', (s_0, 0), G'}$, $(s,j)$, $k$, \textit{visited} \\
	\nl \If{$(s, j) \in \textit{visited}$}{\textbf{return} 0}
	\nl $\textit{visited} \leftarrow \textit{visited} \cup \{(s,j)\}$ \\
	\nl $\textit{cnt} = 0$ \\
	\nl \If{$\pi[(s,j)] = \emptyset$} {
		\textit{// Expand this state for the first time} \\
		\nl \tsc{ff-bellman-update}$\big(\mathbb{S}',(s,j),k\big)$ \\
		\nl \textbf{return} 1 \\
	}
	\nl \ElseIf{$j < k$}{
		\nl \ForAll{$(s',j')~\textit{s.t.}~T'((s',j')|(s,j),\pi[(s,j)]) > 0$}{
			\nl \textit{cnt} += \tsc{ff-expand}$\big(\mathbb{S}', (s,j), k, \textit{visited}\big)$ \\}
	}
	\nl \tsc{ff-bellman-update}$\big(\mathbb{S}',(s,j),k\big)$ \\ 
	\nl \textbf{return} \textit{cnt} \\
	\caption{\tsc{ff-expand}}
	\label{alg:ff-expand}
\end{algorithm}

\begin{algorithm}[t]
	\smaller
	\textbf{input}: $\mathbb{S}'\!=\!\tuple{S', A, T', C', (s_0, 0), G'}$, $(s,j)$, $k$, \textit{visited} \\
	\nl \If{$s \in \textit{visited}$}{\textbf{return} 0}
	\nl $\textit{visited} \leftarrow \textit{visited} \cup \{(s,j)\}$ \\
	\nl $\textit{error} = 0$ \\
	\nl $a \leftarrow \pi[(s,j)]$ \\
	\nl \If{$a = \emptyset$} {
		\textit{the test reached a state that hasn't been expanded yet} \\
		\nl \textbf{return} $\infty$ \\
	}
	\nl \ElseIf{$j < k$}{
		\nl \ForAll{$(s',j')~\textit{s.t.}~T'((s',j')|(s,j),\pi[(s,j)]) > 0$}{
			\nl \textit{error} = $\max\big(\textit{error}$, \tsc{ff-test-convergence}$\big(\mathbb{S}', (s,j), k, \textit{visited}\big)\big)$ \\}
	}
	\nl \textit{error}  \!=\! $\max\!\big(\textit{error}$, \tsc{ff-bellman-update}$\big(\mathbb{S}'\!,(s,j),k\big)\!\big)$\\
	\nl \If{$\pi(s,j) \neq a$}{
		\nl \textbf{return} $\infty$ // \textit{the policy changed} \\}
	\nl \textbf{return} \textit{error} \\
	\caption{\tsc{ff-test-convergence}}
	\label{alg:ff-test}
\end{algorithm}

\begin{algorithm}[t]
	\smaller
	\textbf{input}: $\mathbb{S}'\!=\!\tuple{S', A, T', C', (s_0, 0), G'}$, $(s,j)$, $k$ \\
	\textbf{output}: \textit{error} \\
	\nl $V' \leftarrow V[(s,j)]$ \\
	\nl \If{$j=k$} {
		\nl $D \leftarrow$ \tsc{create-pddl}$(\mathbb{S}',s)$ \\
		\nl $\{s_1, a_1, s_2, a_2, ..., s_L, a_L\} \leftarrow$ \tsc{call-ff}$(D)$ \\
		\nl \For{$1 \leq L$}{
			\nl $V[(s_i, k)] \leftarrow \sum_{i \leq x \leq L} C'((s_x, k), a_i)$ \\			
			\nl $\pi[(s_i, k)] \leftarrow a_i $ \\
		}
	}
	\nl \Else{
		\nl $V[(s,j)] \leftarrow \min_a Q((s,j), a) $ \\
		\nl $\pi[(s,j)] \leftarrow \arg \min_a Q((s,j), a)$ \\
	}
	\nl \textbf{return} $\left| V[(s,j)] - V' \right|$
	\caption{\tsc{ff-bellman-update}}
	\label{alg:ff-bellman}
\end{algorithm} 

\paragraph{Handling plan deviations during execution}
While \tsc{ff-lao*} solves $\MKL{k}{1}$-reductions, the ultimate goal is to solve the SSP from which the reduction is derived from. As mentioned before, we use $\mathbb{S}$ to denote this SSP. It is easy to see that a complete policy for $\mathbb{S'}$ is not necessarily complete for $\mathbb{S}$. Therefore, during execution we need to be able to handle deviations from the plan returned by \tsc{ff-lao*}.

We use a replanning approach to address this issue, \tsc{ff-lao*-replan}, illustrated in Algorithm~\ref{alg:ff-lao-replan}. The idea is simple: during execution, check if the current state has an action already computed with $j=0$. If that's the case, this action is executed (line~7). Otherwise, \tsc{ff-lao*} is called to solve the reduced model with initial state $(s, 0)$~(lines 5-6). \tsc{ff-lao*-replan} receives the choice of determinization as input ($\Delta$), and creates an $\MKL{k}{1}$-reduction accordingly (line~1).

Note that there are other choices for the replanning criterion. For example, checking if there is any $j \in [0;k]$ such that $(s,j) \in \pi$. Another alternative is to keep track of exceptions during execution, and set the value of $j$ accordingly; in this case, $j$ should be set to 0 after re-planning. Other alternatives are possible. We choose the one used by \tsc{ff-lao*-replan} because it is, in principle, the more robust choice, given that we use the maximum ``look-ahead'' every time\footnote{Note that this is not guaranteed to be better than using $j>0$, since pathological scenarios can be created where increasing $k$ leads to worse plans.}. However, if computational efficiency is a concern, other alternatives might be better. We leave a more in depth analysis of these choices for future work.

\begin{algorithm}[t]
	\smaller
	\textbf{input}: $\mathbb{S}\!=\!\tuple{S, A, T, C, (s_0, 0), G}, \Delta, k, \epsilon$ \\
	\nl $\mathbb{S'} \leftarrow$ \tsc{create-reduction}$(\mathbb{S}, \Delta)$ \\
	\nl $s \leftarrow s_0$ \\
	\nl \While {$s \notin G$} {
		\nl \If{$(s, 0) \notin \pi$}{
			\nl \tsc{replace-initial-state}$(\mathbb{S}', (s, 0))$ \\
			\nl \tsc{ff-lao*}($\mathbb{S}', k, \epsilon$) \\
		}
		\nl $s \leftarrow$ \tsc{execute-action}$(s, \pi[(s, 0)])$
	}
	\caption{\tsc{ff-lao*-replan}}
	\label{alg:ff-lao-replan}
\end{algorithm} 

\paragraph{Theoretical considerations} 
We now show conditions under which \tsc{ff-lao*} is guaranteed to succeed. The following definition will be useful: a \emph{proper policy rooted at }$s$ is one that reaches a goal state with probability 1 from every state it can reach from $s$. 

\begin{proposition}
	Given an admissible heuristic for the reduced model~~$\mathbb{S}'$, if~~$\mathbb{S}'$ has at least one proper policy rooted at $(s_0, 0)$, then \tsc{ff-lao*} is guaranteed to find one in finite time. 
\end{proposition}
\begin{proof}
	Whenever \tsc{ff-lao*} expands a state $(s,k)$ and calls \tsc{ff} on this state, if the call succeeds, the states $s_i$, for $i\in[1,...,L]$, that are part of the plan computed by \tsc{ff} essentially become terminal states of the problem, with costs set as in line 6. Since \tsc{FF} is a sub-optimal planner for deterministic problems, we have that $\sum_{i \leq x \leq L} C'((s_x, k), a_i) \geq V[(s_i,k)]$, and thus the values of all other states $(s,j)$, with $j<k$, are guaranteed to be admissible with respect to the new updated value of the added terminal states. Therefore, after every successful call to \tsc{ff}, the resulting set of values and terminal states form a well-defined SSP, which \tsc{LAO*} is able to solve. 
	
	Moreover, in the case that a call to \tsc{ff} fails for some state $\hat{s}$, this state will be assigned an infinite cost, and thus the improved version of \tsc{LAO*} will avoid $\hat{s}$ as long there is some other path to the goal. Because \tsc{ff} is complete, any state belonging to a proper policy will be assigned a positive cost, so $\hat{s}$ couldn't have been part of a proper policy for $M$. Thus, under the conditions of the theorem, every call to \tsc{ff} transforms the problem becomes into an MDP with avoidable dead-ends~\cite{Kolobov12}, which \tsc{lao*} is able to solve\footnote{While this is not true for the original version of \tsc{lao*}, this is true of the so-called \emph{improved} version of \tsc{lao*} that we use in this work, which performs Bellman backups of states in depth-first fashion, in post order traversal.}.
\end{proof}

Unfortunately, as is the case for virtually all replanning algorithms, not much can be guaranteed about the quality of plans found by \tsc{ff-lao*-replan} for $\mathbb{S}$. However, as we show in our experiments, by carefully choosing the input determinization, $\Delta$ and the bound $k$, \tsc{ff-lao*-replan} can find successful policies extremely quickly, even in domains well-known for their computational hardness and the presence of dead-end states. 

\section{Choosing a Good Determinization}
\label{sec:choosing-red}
Many stochastic domains have an inherent structure that make some of their determinizations significantly more effective than others. Consider, for instance, the triangle-tireworld domain~\cite{little2007probabilistic}. The agent has to reach one of the vertices in a planar graph of triangular shape, but after every move there is the possibility of getting a flat tire (see Figure~\ref{fig:triangle}). If this happens, it must get a spare tire before being able to move again. However, spares are only available in certain locations, and there is only a single path from the start to the goal such that all locations in the path have spares. This domain has two possible determinizations, depending on whether a flat tire happens after moving or not. As it turns out, it is possible to get the optimal policy for this problem by planning as if a flat tire will always occur. The interesting part is that this is true for \emph{all} instances of this problem, regardless of size.

\begin{figure}
	\centering
	\includegraphics[width=1.6in]{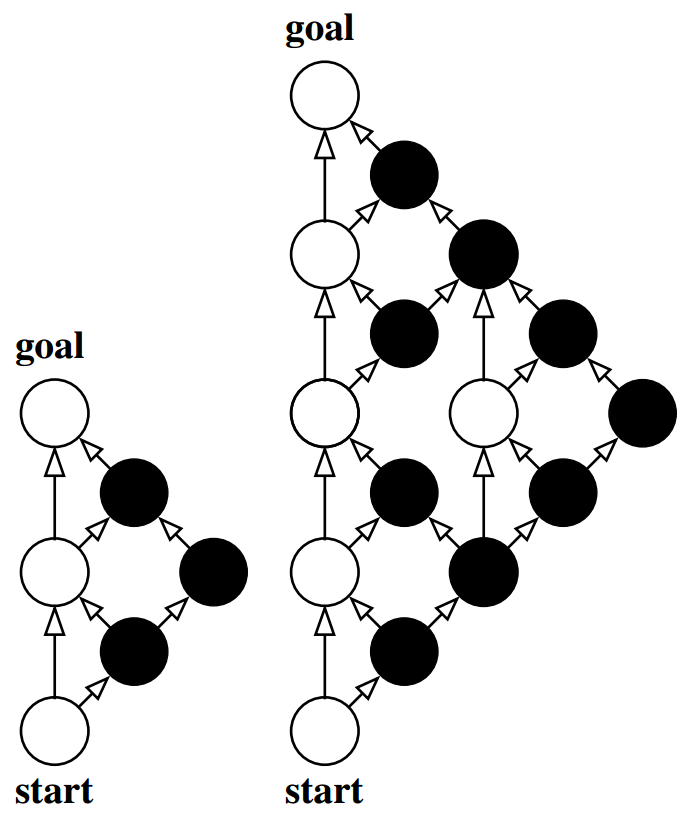}
	\caption{Two instances of the \tsc{triangle-tireworld} domain. Locations with spare tires are marked in black.}
	\label{fig:triangle}
\end{figure}

Triangle-tireworld is a great example of a domain where all domain instances share a probabilistic structure that can be captured by determinization. In practical terms, this means that it is possible to learn a determinization on the smaller problems, and then use it for solving larger ones. 

We adopt this approach for choosing the input determinization to \tsc{ff-lao*-replan}, $\Delta$. We assume a set of problems of varying sizes are available and that it is easy to identify the smaller ones. This is the case for the IPPC benchmarks that we consider in our experiments, as they are typically ordered by problem size/difficulty. However, in general, some analysis of the problem might be required, for example, by counting the number of possible grounded atoms that the problem description induces. Another possibility is to generate small problems automatically from the domain description. This has the advantage that it doesn't require additional problems for learning the determinization to use (besides the actual problem to solve). Although this is certainly an interesting possibility, it requires some finesse and it is outside the scope of this work. 

Algorithm~\ref{alg:choose-det} illustrates \tsc{learning-det}, a brute-force approach to learn a determinization $\Delta$ for domain $\mc{D}$; $\mathbb{S}_l$ represents the problem used for learning. This procedure does a comprehensive search over the space of all determinizations. For each, we estimate the probability of success ($P_i$) and the expected execution cost ($\mathbb{C}_i$) of running \tsc{ff-lao*-replan} on $\mathbb{S}_l$; the costs are estimated using Monte-Carlo simulations. Finally, we pick the determinization with the lowest expected cost, among the ones with the highest probability of success. 

\begin{algorithm}[t]
	\smaller
	\textbf{input}: $\mc{D}, \mathbb{S}_l, k$ \\
	\textbf{output}: $\Delta$ \\
	\nl $\{\Delta_1, ..., \Delta_M\} \leftarrow$ Create all possible determinizations of $\mc{D}$\\
	\nl \ForAll{$i \in \{1,...,M\}$} {
		\nl $P_i, \mathbb{C}_i \leftarrow$ Estimate probability of successs and expected cost of executing \tsc{ff-lao*-replan}($\mathbb{S}_l, \Delta_i, k$) \\
	}
	\nl $P^* \leftarrow \max_i P_i $ \\
	\nl $\Delta \leftarrow \Delta_{\min_i \mathbb{C}_i~\textit{s.t.}~P_i=P^*}$ \\
	\caption{\tsc{learning-det}}
	\label{alg:choose-det}
\end{algorithm} 

There are some subtleties involved in this process. Note that \tsc{ff-lao*} is only guaranteed to terminate if the input reduction has a proper policy from its initial state. This will most likely not be the case for many of the determinizations explored by \tsc{learning-det}; in fact, under some determinizations the goals might be completely unreachable from any state. 

We address this issue using a well-known technique for planning with problems involving dead-ends. In particular, we use a cap $M$ on state costs, including the costs assigned when \tsc{ff} fails, and modify the Bellman backup operator used by \tsc{ff-lao*} as
\begin{equation*}
V(s) = \min \Big\{ M, \min_{a \in A} \Big \{ C(s,a) + \sum_{s' \in S} T(s'|s,a)V(s') \Big \} \Big\} 
\label{eq:bellman-cap}
\end{equation*}
which guarantees the convergence of heuristic search algorithms~\cite{Kolobov12}. While this introduces a new parameter impacting the planner's decisions, and hides the true impact of dead-end states. Note that \tsc{learning-det} still attempts to maximize the multi-objective evaluation criterion typically used when unavoidable dead-ends exist~\cite{Kolobov12,steinmetz2016revisiting}. 

\section{Experiments}
\label{sec:experiments}
\subsubsection{Domains and methodology}
We evaluated \tsc{ff-lao*} and \tsc{learning-det} on a set of problems taken from IPPC'08~\cite{bryce20086th}. Specifically, we used the first 10 problem instances of the following four domains: \tsc{triangle-tireworld}, \tsc{blocksworld}, \tsc{ex-blocksworld}, and \tsc{zenotravel}. Unfortunately, the rest of the IPPC'08 domains are not supported by our PPDDL parser \cite{bonet2005mgpt}. Additionally, we modified the \tsc{ex-blocksworld} domain to avoid the possibility of blocks to be put on top of themselves~\cite{trevizan2014depth}. 

The evaluation methodology was similar to the one used in past planning competitions: we give each planner 20 minutes to solve 50 rounds of each problem (i.e., reach a goal state starting from the initial state). Then we measure its performance in terms of the number of rounds that the planner was able to solve during that time. All experiments were conducted on an Intel Core i7-6820HQ machine running at 2.70GHz with a 4GB memory cutoff. 

We evaluated the planners using the \mbox{MDPSIM}~\cite{younes2005first} client/server program for simulating SSPs, by having planners repeatedly perform the following three steps: 1) connect to the MDPSIM server to receive a state, 2) compute an action for the received state and send the action to the MDPSIM server, and 3) wait for the server to simulate the result of applying this action and send a new state. A simulation ends when a goal state is reached, when an invalid action is sent by the client, or after 2500 actions have been sent by the planner.

We compared the performance of \tsc{ff-lao*} with our own implementations of \tsc{ff-replan} and \tsc{rff}, as well as the original author's implementation of \tsc{ssipp}~\cite{trevizan2014depth}. We evaluated two variants of \tsc{ff-replan}, one using MLO (\tsc{ff$_s$}) and another one using AO (\tsc{ff$_a$}). For \tsc{rff} we used MLO and the \emph{Random Goals} variant, in which before every call to \tsc{ff}, a random subset (size 100) of the previously solved states are added as goal states. Additionally, we used a probability threshold $\rho=0.2$. The choice of these parameters was informed by analysis in the original work~\cite{teichteil2010incremental}. For \tsc{ssipp} we used $t=3$ and the $h_{add}$ heuristic, parameters also informed by the original work~\cite{trevizan2014depth}.

For \tsc{ff-lao*}, we learned a good determinization to use by applying \tsc{learning-det} on the first problem of each domain (p01), with $k=0$. This choice of $k$ was motivated both by time considerations, and by the rationale that $k=0$ should better reflect the impact of each determinization (since \tsc{ff-lao*} becomes a fully determinization-based planner). We used a dead-end cap $\mathcal{D}=500$ throughout our experiments. We initialized values with the non-admissible FF heuristic~\cite{bonet2005mgpt}. 

We ran \tsc{learning-det} offline, prior to the \mbox{MDPSIM} evaluation. Note, however, that the time taken by the brute force search plus the time used to solve problem  p01 with the chosen determinization was, in all cases, well below the 20 minutes limit (approx.~2 minutes in the worst case).
The remaining parameter for \tsc{ff-lao*} is the value of $k$. We report the best performing configuration in the range $k\in[0,3]$, which was $k=0$ for most domains, with the exception of \tsc{ex-blocksworld}, which required $k=3$. Note that \tsc{ff-lao*} with $k=0$ is essentially equivalent to \tsc{ff-replan}, so any advantage obtained over \tsc{ff$_s$} and \tsc{ff$_a$} is completely derived from the choice of determinization. 

\subsubsection{Results and Discussion}
Figure~\ref{fig:results} shows the number of successful rounds obtained by each planner in the benchmarks. In general, \tsc{ff-lao*} either tied for the best, or outperformed the baselines. All planners had a 100\% success rate in \tsc{blocksworld}, so there is not much room for comparison.

In the \tsc{triangle-tireworld} domain, \tsc{ff-lao*} and \tsc{ff$_s$} had 100\% success rate, while \tsc{rff} ran out of time in the last 3 problems. On the other hand, the performance of \tsc{ssipp} and \tsc{ff$_a$} deteriorated quickly as the problem instance increased. It is worth pointing out that the performances of \tsc{ff$_s$} and \tsc{rff} in this domain are quite sensitive to tie-breaking---there are only two outcomes to choose from, each occurring with 0.5 probability. As the results of \tsc{ff$_a$} suggest, a different choice would have resulted in a much worse success rate. On the other hand, the use of \tsc{learning-det} gets around this issue by automatically choosing the best determinization to use, a process that took seconds. While we do note that the \emph{best goals} parameterization of \tsc{rff} gets around this issue, its computational cost is much harder, so it's not obvious that it would actually improve performance in this case~\cite{teichteil2010incremental}.

\begin{figure*}
	\center
	\hspace{-5pt}\includegraphics[width=6.1in]{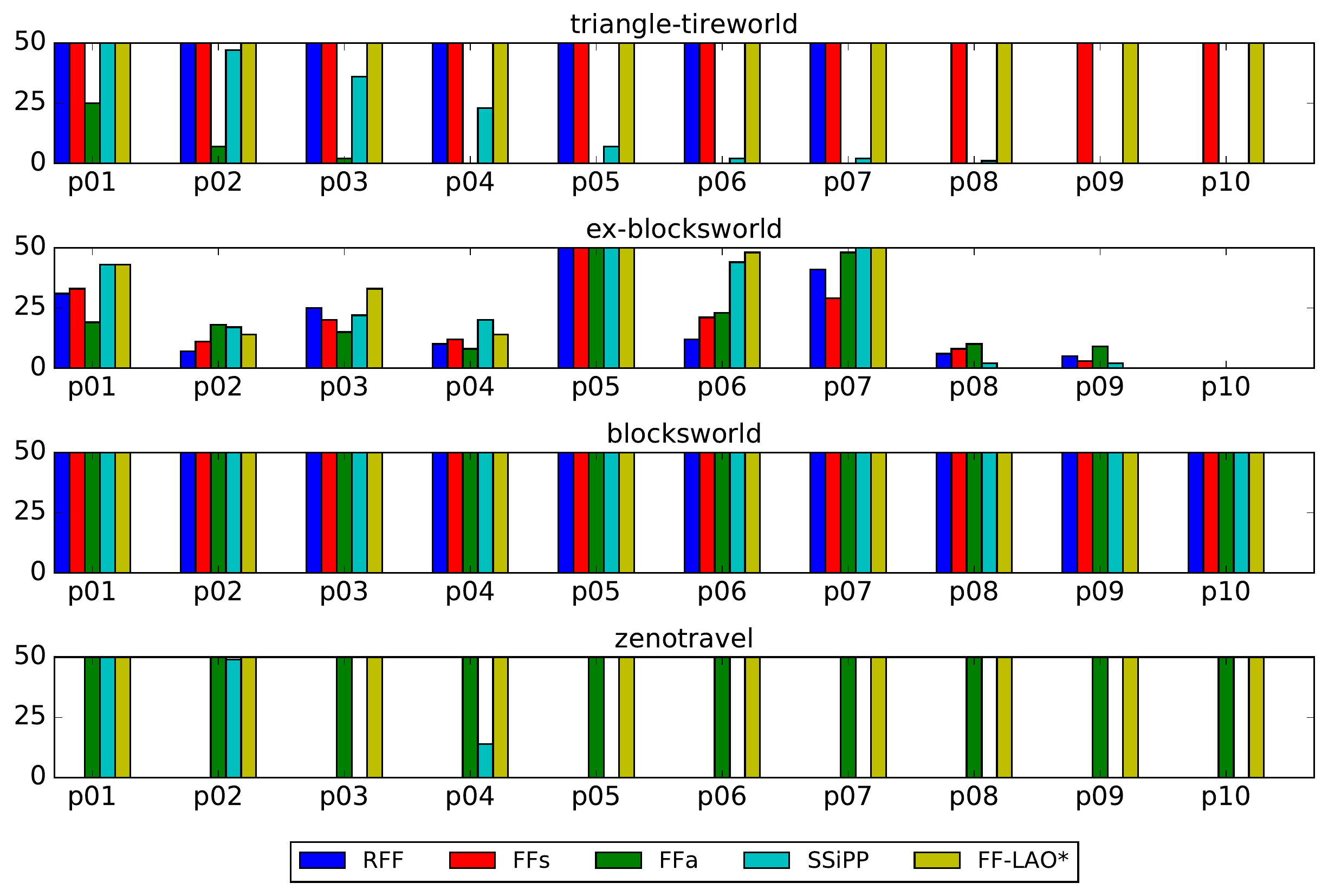}
	\caption{Number of solved rounds by 5 different planners in IPPC'08 benchmarks.}
	\label{fig:results}
\end{figure*} 

In the \tsc{ex-blocksworld} domain \tsc{ff-lao*} (with $k=3$) and \tsc{ssipp} significantly outperform the other two planners, solving 252 and 250 rounds, respectively, against 187 for both \tsc{ff$_s$} and \tsc{rff}, and 200 for \tsc{ff$_a$}. Interestingly, in this domain the determinization found by \tsc{learning-det} is not sufficient to obtain good performance; in fact, only 3 problems had a non-zero success rate with $k=0$. This highlights the utility of doing probabilistic reasoning with \tsc{ff-lao*}. Although not shown here for space considerations, the performance with $k=1$ (214 successful rounds) was already better than all the baselines, except for \tsc{ssipp}.

In \tsc{zenotravel}, \tsc{ff-lao*} and \tsc{ff$_a$} were remarkably better than the other two planners: they achieved 100\% success rate in all domain instances, while the other baselines failed almost all of the rounds. In the case of the determinization-based planners, this is due to the goal becoming unreachable under MLO, so the choice of determinization has a significant impact on performance. 

\section{Conclusion}
\label{sec:conclusion}
In this work we presented a novel perspective on the use of determinization for probabilistic planning, by showing that a careful choice of determinization can outperform state-of-the-art planners. We also introduced a new planner, \tsc{ff-lao*}, that, given a choice of determinization, leverages the power of classical planning algorithms for computational efficiency, but can also reason probabilistically if desired.

We proposed a strategy for selecting a determinization that takes advantage of the inherent structure of a given stochastic domain. We show that the choice of determinization can generalize across problems of varying size, in particular, in terms of its impact on planning performance (probability of success and expected cost).

We compared our approach to state-of-the-art planners for goal-oriented probabilistic problems, using a set of benchmarks taken from the International Probabilistic Planning Competition. Our results strongly support the claim that the choice of determinization can lead to very substantial gains in performance, and position \tsc{ff-lao*}---using our determinization learning approach---as a competitive planner for large stochastic domains.

In future work, we plan to explore ways to vary the choice of determinization according to state-space features, and to develop automated methods for generating problems useful for learning a good determinization of a given domain.

\section{Acknowledgment}
This work was supported in part by NSF Grant No. 1524797.

\balance
\bibliographystyle{named}
\bibliography{rbr-luis}

\begin{thebibliography}{}

\bibitem[\protect\citeauthoryear{Bertsekas and
  Tsitsiklis}{1991}]{bertsekas1991analysis}
Dimitri~P. Bertsekas and John~N. Tsitsiklis.
\newblock An analysis of stochastic shortest path problems.
\newblock {\em Mathematics of Operations Research}, 16(3):580--595, 1991.

\bibitem[\protect\citeauthoryear{Bertsekas and
  Tsitsiklis}{1995}]{bertsekas1995neuro}
Dimitri~P. Bertsekas and John~N. Tsitsiklis.
\newblock Neuro-dynamic programming: an overview.
\newblock In {\em Proceedings of the 34th IEEE Conference on Decision and
  Control}, pages 560--564, 1995.

\bibitem[\protect\citeauthoryear{Bonet and Geffner}{2005}]{bonet2005mgpt}
Blai Bonet and H{\'e}ctor Geffner.
\newblock {mGPT}: A probabilistic planner based on heuristic search.
\newblock {\em Journal of Artificial Intelligence Research}, 24:933--944, 2005.

\bibitem[\protect\citeauthoryear{Bryce and Buffet}{2008}]{bryce20086th}
Daniel Bryce and Olivier Buffet.
\newblock Sixth international planning competition: Uncertainty part.
\newblock In {\em Proceedings of the Sixth International Planning Competition
  (IPC'08)}, 2008.

\bibitem[\protect\citeauthoryear{Hajian \bgroup \em et al.\egroup
  }{2016}]{hajian2016modeling}
Mohammad Hajian, Emanuel Melachrinoudis, and Peter Kubat.
\newblock Modeling wildfire propagation with the stochastic shortest path: A
  fast simulation approach.
\newblock {\em Environmental Modelling \& Software}, 82:73--88, 2016.

\bibitem[\protect\citeauthoryear{Hansen and Zilberstein}{2001}]{HZaij01}
Eric~A. Hansen and Shlomo Zilberstein.
\newblock {LAO}*: A heuristic search algorithm that finds solutions with loops.
\newblock {\em Artificial Intelligence}, 129(1-2):35--62, 2001.

\bibitem[\protect\citeauthoryear{Hoffmann and Nebel}{2001}]{hoffmann2001ff}
J{\"o}rg Hoffmann and Bernhard Nebel.
\newblock The {FF} planning system: Fast plan generation through heuristic
  search.
\newblock {\em Journal of Artificial Intelligence Research}, 14(1):253--302,
  2001.

\bibitem[\protect\citeauthoryear{Keyder and Geffner}{2008}]{keyder2008hmdpp}
Emil Keyder and Hector Geffner.
\newblock The {HMDPP} planner for planning with probabilities.
\newblock In D.~Bryce and O.~Buffet, editors, {\em ICAPS Third International
  Probabilistic Planning Competition}. IPPC'08, 2008.

\bibitem[\protect\citeauthoryear{Kolobov \bgroup \em et al.\egroup
  }{2012}]{Kolobov12}
Andrey Kolobov, Mausam, and Daniel~S. Weld.
\newblock A theory of goal-oriented {MDP}s with dead ends.
\newblock In {\em Proceedings of the Conference on Uncertainty in Artificial
  Intelligence}, pages 438--447, Catalina Island, California, 2012.

\bibitem[\protect\citeauthoryear{Lim \bgroup \em et al.\egroup
  }{2013}]{lim2013practical}
Sejoon Lim, Christian Sommer, Evdokia Nikolova, and Daniela Rus.
\newblock Practical route planning under delay uncertainty: Stochastic shortest
  path queries.
\newblock In {\em Proceedings of Robotics: Science and Systems}, volume~8,
  pages 249--256, 2013.

\bibitem[\protect\citeauthoryear{Little and
  Thiebaux}{2007}]{little2007probabilistic}
Iain Little and Sylvie Thiebaux.
\newblock Probabilistic planning vs. replanning.
\newblock In {\em Proceedings of the ICAPS'07 Workshop on the International
  Planning Competition: Past, Present and Future}, 2007.

\bibitem[\protect\citeauthoryear{Littman}{1997}]{littman1997probabilistic}
Michael~L. Littman.
\newblock Probabilistic propositional planning: Representations and complexity.
\newblock In {\em Proceedings of the Fourteenth National Conference on
  Artificial Intelligence}, pages 748--754, Providence, Rhode Island, 1997.

\bibitem[\protect\citeauthoryear{Pineda and Zilberstein}{2014}]{PZicaps14}
Luis Pineda and Shlomo Zilberstein.
\newblock Planning under uncertainty using reduced models: Revisiting
  determinization.
\newblock In {\em Proceedings of the 24th International Conference on Automated
  Planning and Scheduling}, pages 217--225, Portsmouth, New Hampshire, 2014.

\bibitem[\protect\citeauthoryear{Puterman}{1994}]{Puterman94}
Martin~L. Puterman.
\newblock {\em Markov Decision Processes: Discrete Stochastic Dynamic
  Programming}.
\newblock John Wiley \& Sons, Inc., New York, NY, USA, 1994.

\bibitem[\protect\citeauthoryear{Steinmetz \bgroup \em et al.\egroup
  }{2016}]{steinmetz2016revisiting}
Marcel Steinmetz, Joerg Hoffmann, and Olivier Buffet.
\newblock Revisiting goal probability analysis in probabilistic planning.
\newblock In {\em In Proceedings of the 26th International Conference on
  Automated Planning and Scheduling}, 2016.

\bibitem[\protect\citeauthoryear{Tan \bgroup \em et al.\egroup
  }{2015}]{tan2015stochastic}
Xiaoqi Tan, Yuan Wu, and Danny~H.K. Tsang.
\newblock A stochastic shortest path framework for quantifying the value and
  lifetime of battery energy storage under dynamic pricing.
\newblock {\em IEEE Transactions on Smart Grid}, pages 769--778, 2015.

\bibitem[\protect\citeauthoryear{Teichteil-K{\"o}nigsbuch \bgroup \em et
  al.\egroup }{2010}]{teichteil2010incremental}
Florent Teichteil-K{\"o}nigsbuch, Ugur Kuter, and Guillaume Infantes.
\newblock Incremental plan aggregation for generating policies in {MDP}s.
\newblock In {\em Proceedings of the Ninth International Conference on
  Autonomous Agents and Multiagent Systems}, pages 1231--1238, Toronto, Canada,
  2010.

\bibitem[\protect\citeauthoryear{Trevizan and Veloso}{2014}]{trevizan2014depth}
Felipe~W Trevizan and Manuela~M Veloso.
\newblock Depth-based short-sighted stochastic shortest path problems.
\newblock {\em Artificial Intelligence}, 216:179--205, 2014.

\bibitem[\protect\citeauthoryear{Wray \bgroup \em et al.\egroup
  }{2016}]{wray2016hierarchical}
Kyle~Hollins Wray, Luis Pineda, and Shlomo Zilberstein.
\newblock Hierarchical approach to transfer of control in semi-autonomous
  systems.
\newblock In {\em Proceedings of the 2016 International Conference on
  Autonomous Agents \& Multiagent Systems}, pages 1285--1286, 2016.

\bibitem[\protect\citeauthoryear{Yoon \bgroup \em et al.\egroup
  }{2007}]{yoon2007ff}
Sung~Wook Yoon, Alan Fern, and Robert Givan.
\newblock {FF}-{R}eplan: A baseline for probabilistic planning.
\newblock In {\em Proceedings of the Seventeenth International Conference on
  Automated Planning and Scheduling}, pages 352--359, Providence, Rhode Island,
  2007.

\bibitem[\protect\citeauthoryear{Yoon \bgroup \em et al.\egroup
  }{2008}]{yoon2008probabilistic}
Sungwook Yoon, Alan Fern, Robert Givan, and Subbarao Kambhampati.
\newblock Probabilistic planning via determinization in hindsight.
\newblock In {\em Proceedings of the Twenty-Third National Conference on
  Artificial Intelligence}, pages 1010--1016, Chicago, Illinois, 2008.

\bibitem[\protect\citeauthoryear{Younes and Littman}{2004}]{younes2004ppddl1}
H{\aa}kan L.~S. Younes and Michael~L. Littman.
\newblock {PPDDL}1.0: An extension to {PDDL} for expressing planning domains
  with probabilistic effects.
\newblock {\em Technical Report CMU-CS-04-162}, 2004.

\bibitem[\protect\citeauthoryear{Younes \bgroup \em et al.\egroup
  }{2005}]{younes2005first}
H{\aa}kan L.~S. Younes, Michael~L. Littman, David Weissman, and John Asmuth.
\newblock The first probabilistic track of the international planning
  competition.
\newblock {\em Journal of Artificial Intelligence Research}, 24(1):851--887,
  2005.

\end{thebibliography}

\end{document}